\documentclass[10pt,twocolumn,letterpaper]{article}
\usepackage{float}
\usepackage{bm}
\usepackage{amsthm}
\usepackage{multirow} 
\usepackage{makecell}
\usepackage{bbding}
\usepackage{threeparttable}
\usepackage{tabularx}
\usepackage[table]{xcolor} 
\usepackage{ulem}

\usepackage{cvpr}              
\definecolor{gray50}{gray}{0.5}

\newcommand{\mathcalbf}[1]{\bm{\mathcal{#1}}}

\definecolor{cvprblue}{rgb}{0.21,0.49,0.74}
\usepackage[pagebackref,breaklinks,colorlinks,allcolors=cvprblue]{hyperref}

\def\paperID{6630} 
\def\confName{CVPR}
\def\confYear{2026}

\title{RaUF: Learning the Spatial Uncertainty Field of Radar}

\usepackage{pifont}
\usepackage{pgf,colortbl}

\newcommand{\acrot}[1]{
 \pgfmathsetmacro\percentcolor{max(0,min(100,#1/23.0*100))} 
  \edef\tempcolor{\noexpand\cellcolor{red!\percentcolor!white}}
  \tempcolor #1 
}
\newcommand{\actrans}[1]{
 \pgfmathsetmacro\percentcolor{max(0,min(100,#1/2.0*100))}
  \edef\tempcolor{\noexpand\cellcolor{blue!\percentcolor!white}}
  \tempcolor #1 
}

\author{Shengpeng Wang \qquad Kuangyu Wang\\
Huazhong University of Science and Technology\\
{\tt\small wsp666@hust.edu.cn \qquad wangky@hust.edu.cn}
\and
Wei Wang\thanks{Corresponding author}\\
Wuhan University\\
{\tt\small wangw@whu.edu.cn}
}

\newtheorem{theorem}{Theorem}

\begin{document}

\maketitle


\begin{abstract}
Millimeter-wave radar offers unique advantages in adverse weather but suffers from low spatial fidelity, severe azimuth ambiguity, and clutter-induced spurious returns. Existing methods mainly focus on improving spatial perception effectiveness via coarse-to-fine cross-modal supervision, yet often overlook the ambiguous feature-to-label mapping, which may lead to ill-posed geometric inference and pose fundamental challenges to downstream perception tasks. In this work, we propose RaUF, a spatial uncertainty field learning framework that models radar measurements through their physically grounded anisotropic properties. To resolve conflicting feature-to-label mapping, we design an anisotropic probabilistic model that learns fine-grained uncertainty. To further enhance reliability, we propose a Bidirectional Domain Attention mechanism that exploits the mutual complementarity between spatial structure and Doppler consistency, effectively suppressing spurious or multipath-induced reflections. Extensive experiments on public benchmarks and real-world datasets demonstrate that RaUF delivers highly reliable spatial detections with well-calibrated uncertainty. Moreover, downstream case studies further validate the enhanced reliability and scalability of RaUF under challenging real-world driving scenarios. 
\end{abstract}    
\section{Introduction}
\label{sec:intro}

Radar perception~\cite{wang2025mitigating,zhang2025rald,ijcai2025p979} has become an essential component in various applications, including autonomous driving~\cite{kim2023crn,gao2022dc}, robotics~\cite{wang2025s3e,harlow2024new}, and surveillance systems. Its ability to operate effectively in adverse weather conditions and low-visibility environments makes all-weather, all-day perception possible while enhancing situational awareness and safety. Moreover, the unique Doppler information embedded in radar signals provides the relative radial velocities of scatterers, offering great potential to be exploited for localization~\cite{wang2025s3e}, velocity estimation~\cite{pang2024radarmoseve}, object detection~\cite{lin2024rcbevdet}, and segmentation tasks as a complementary cue to vision-based perception. However, the sparsity and noise inherent in radar data present significant challenges for accurate dense perception. 

\begin{figure}[t]
  \centering
   \includegraphics[width=\linewidth]{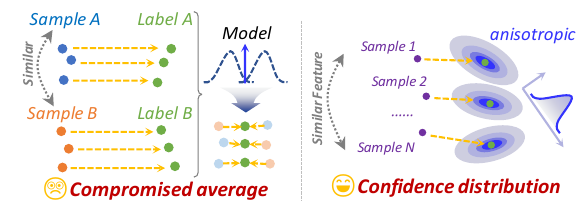}

    \caption{Traditional methods compromise under conflict, leading to \textit{ill-posed geometric inference}.}
   \label{level.sample}
\end{figure}
To improve radar perception, existing studies~\cite{fan2024enhancing,gao2025radar,ijcai2025p979} focus on enhancing radar point clouds~(PCLs) via coarse-to-fine cross-modal supervision from higher-resolution sensors, such as cameras and LiDARs. However, two major challenges persist: \textit{1) Ill-posed geometric inference from sparse measurements}. The detail refinement from low- to high-resolution is essentially achieved by the network “hallucinating” geometric structures from sparse radar cues, which may lack physical fidelity. Such \textit{ambiguous feature-to-label mappings} can induce supervisory inconsistency across samples. As shown in Fig.~\ref{level.sample}, in the absence of uncertainty modeling, the network is forced to reconcile conflicting supervisory signals, often converging toward geometrically implausible intermediate solutions that reflect neither mode of the true distribution. Such a compromised average ultimately impedes stable optimization and feature generalization. \textit{2) Overreliance on amplitude cues without robustness to spurious reflections}, many current approaches heavily rely on intensity values for radar perception, often overlooking the presence of ghost points—false detections that can arise from multipath reflections or noise. This oversight can lead to unreliable perception results owing to the misinterpretation of radar data.

To address these challenges, we delve deeper into the physical nature of radar and derive a physically grounded motivation for our framework. As illustrated in Fig.~\ref{level.uncertainty}, we observe an inherent anisotropic measurement property of radar: due to the limited number of effective angle-of-arrival antennas, the azimuth uncertainty is significantly higher than range, resulting in a characteristic ``crescent-shaped" spatial distribution. We further claim the effectiveness of uncertainty-aware learning, which transforms ambiguous feature-to-label mappings from conflicting supervision into informative cues for learning fine-grained confidence over target regions, thereby enhancing both generalization and physical interpretability. Furthermore, we exploit the Doppler consistency inherent for positive detections as shown in Fig.~\ref{level.doppler}. Doppler velocity of the reflection is determined by the radar’s ego-velocity and the scatter’s directional vector. The temporal coherence of Doppler measurements serves as a physically reliable cue for identifying and suppressing spurious or multipath-induced reflections, thus improving the robustness and reliability of spatial perception under cluttered environments.

To fully unlock such inspiring potential, we propose \textit{RaUF}, a novel doppler-aware spatial perception method, which learns the spatial anisotropic uncertainty field of radar measurements and effectively alleviates ambiguities in geometric inference. Specifically, we first introduce an anisotropic geometric uncertainty representation for radar target regions and formulate a Bayesian probabilistic loss that explicitly incorporates such geometric uncertainty. This design transforms the conflicting supervision among ambiguous feature-to-label mappings into a meaningful learning signal, enabling the network to infer fine-grained confidence distributions over target regions rather than relying on incomplete deterministic mappings. This enables the network to learn fine-grained confidence over target regions, thereby enhancing both generalization and physical interpretability. Next, we introduce a \textit{bidirectional domain attention} on spatial and Doppler features to identify and suppress spurious reflections, thus improving the robustness and reliability of spatial perception under cluttered environments. Finally, we conduct extensive experiments across public and real-world datasets and various types of radar to confirm the effectiveness, reliability, and scalability of our approach on downstream tasks.

To sum up, our primary contributions are evident in the following aspects.
\begin{itemize}
    \item We introduce a novel radar perception framework that learns anisotropic geometric uncertainty calibration to resolve the inherently ill-posed geometric inference of low-fidelity radar space. To the best of our knowledge, this is the first attempt to achieve spatial uncertainty learning for radar.
    \item We develop a bidirectional domain attention module for Doppler radar that  reinforces spatial features with Doppler information, thereby reducing false positives with a physically grounded mechanism.
    \item We conduct extensive experiments on benchmark datasets, demonstrating the superiority of our approach. Additionally, we will make our self-collected dataset publicly available to the research community.
\end{itemize}

\begin{figure*}[h]
  \centering
  \subfloat[]{
    \label{level.uncertainty}
    \includegraphics[width=0.645\linewidth]{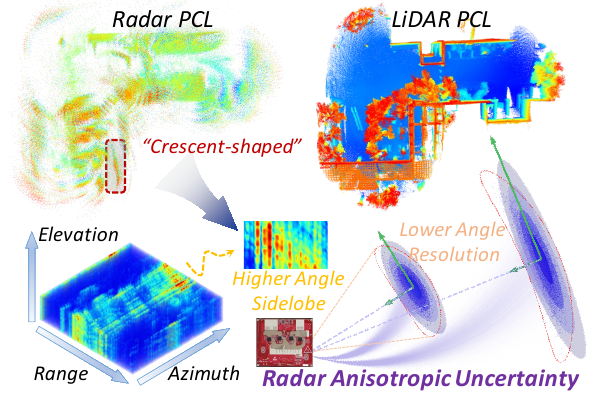}}
  \subfloat[]{
    \label{level.doppler}
    \includegraphics[width=0.35\linewidth]{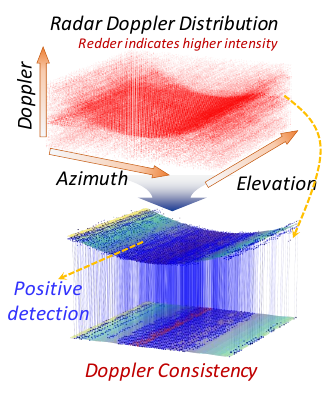}}
   \caption{Our insights:~(a) Multi-frame statistical analysis shows that radar PCLs inherently exhibit anisotropic uncertainty, unlike the isotropic nature of LiDAR. (b) The Doppler measurements of positive detections conform to kinematic constraints.~(Experimental data are based on Coloradar~\cite{kramer2022coloradar} and RaDelft~\cite{iroldan2024})}
   \label{fig:motivation}
\end{figure*}

\section{Related Work}
\subsection{Point Cloud Extraction for Radar}
Conventional radar point cloud extraction mainly relies on signal processing methods such as Constant-False-Alarm-Rate~(CFAR)~\cite{nitzberg1972constant} and its variants~\cite{blake1988cfar,gandhi2002optimality}. However, constrained by the limited number of antennas, the resulting point clouds are typically sparse and noise-prone, insufficient for dense perception tasks like velocity estimation, environment reconstruction, and localization. Although virtual antenna array designs~\cite{wang2009high,zheng2024enhancing,dodds2025non} have significantly improved angular resolution, they often suffer from hardware incompatibility or strong dependence on motion priors. 

Recently, growing efforts have shifted toward generating high-fidelity, dense radar point clouds through coarse-to-fine cross-modal supervision from higher-resolution sensors. RPDNet~\cite{cheng2022novel} employs deep networks to refine Range-Doppler spectrograms under LiDAR supervision, yet remains constrained by the CFAR-style processing paradigm that limits angular resolution. RadarHD~\cite{prabhakara2023high} projects multi-frame CFAR-sparse point clouds onto planar Range-Azimuth bird's-eye view maps and employs a U-Net~\cite{ronneberger2015u} to learn structured occupancy representations. More recently, Radar-Diffusion~\cite{zheng2024enhancing} and RaLD~\cite{zhang2025rald} introduce diffusion-based spectral denoising to enhance spatial resolution, but their heavy reliance on intensity cues leaves them vulnerable to artifacts and multipath distortions. SDDiff~\cite{ijcai2025p979} jointly estimates point clouds and ego velocity for single-chip radar, achieving complementary benefits across tasks. Yet, its dependence on computationally repeated diffusion iterations hampers scalability and onboard deployment in cascaded radar systems.

\subsection{Point Cloud Uncertainty Quantification}
Some work such as GICP~\cite{koide2021voxelized} could quantify local spatial variance via statistical point distributions. Such uncertainty descriptions capture in-domain spatial variance within dense geometric structures but are ill-suited for the sparse and noisy radar points. S3E~\cite{wang2025s3e} and UTR~\cite{barnes2020under} estimate uncertainty from attention-weighted spectral features and data-association, where occupancy confidence is inferred from isotropic score. However, these models fail to capture the anisotropic physical characteristics of radar that stem from distinct range-angle resolution disparities. Inspired by structure-from-motion and pose-graph optimization, radar SLAM systems~\cite{zhuang20234d,gao2022dc,gao2023robust} introduce back-end uncertainty propagation, modeling detection-level landmark uncertainties through pose refinement. But they require multi-frame consistency and precise landmark correspondence. 
\section{Methodology}
In this section, we first introduce the Bayesian Probabilistic Model~(BPM) for radar uncertainty learning in Sec.~\ref{BPM}. Then we present Doppler-Aware Predictive Enhancement to exploit Doppler consistency for improving radar geometric representation in Sec.~\ref{sub.DAPM}. followed by the architecture and uncertainty training details of our proposed RaUF as illustrated in Sec.~\ref{sub.Network}. The framework is shown in Fig.~\ref{fig:overview}.

\begin{figure*}[h]
  \centering
   \includegraphics[width=\linewidth]{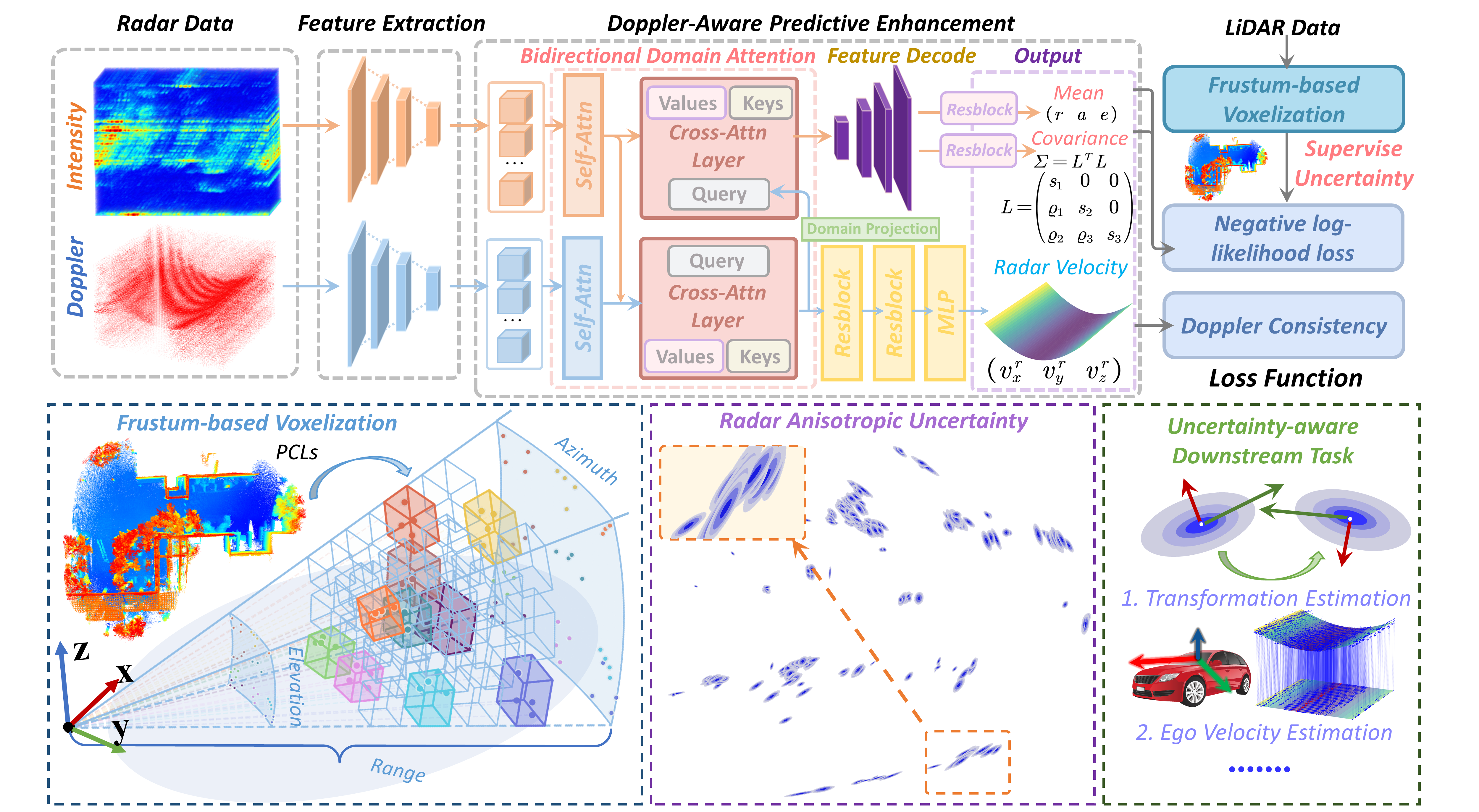}

   \caption{Framework of our RaUF for learning radar anisotropic uncertainty with doppler-aware predictive enhancement. }
   \label{fig:overview}
\end{figure*}

\subsection{BPM for Radar Uncertainty Learning}\label{BPM}
To mitigate ill-posed geometric inference from radar sparse measurements, we formulate Bayesian Probabilistic Model~(BPM) to learn radar spatial uncertainty. We redefine the task as estimating the localization predictive model $f_{\bm\theta}(x)$ and uncertainty quantification model $g_{\bm\varphi}(x)$ parameterized by $\bm\theta$ and $\bm\varphi$, respectively. $x\in \mathbb{R}^{R\times A\times E \times 2}$ is the radar measurement input, where $R$, $A$, and $E$ denote the range, azimuth, and elevation dimensions, respectively, which are supervised by occupancy ground truth $y$. To exploit Doppler consistency for suppressing spurious reflections, we concatenate intensity and Doppler domains into the feature channel. Under BPM we specify a likelihood $p(y|\bm\theta, \bm\varphi,x)$ to model the data generation process. Learning $(\bm\theta,\bm\varphi)$ is equivalent to minimizing the negative log-posterior:
\begin{equation}
	\begin{aligned}
\left( \hat{\boldsymbol{\theta}},\hat{\boldsymbol{\varphi}} \right) &=\mathrm{arg}\min_{\boldsymbol{\theta },\boldsymbol{\varphi }} \left[ -\log p(\boldsymbol{\theta },\boldsymbol{\varphi }\mid \mathcalbf{D} ) \right] \\
&=\mathrm{arg}\min_{\theta} \left[ -\log p(\mathcalbf{D} \mid \boldsymbol{\theta },\boldsymbol{\varphi })-\log p(\boldsymbol{\theta },\boldsymbol{\varphi }) \right]
\end{aligned}
\end{equation}
where $\mathcalbf{D}=\{ (x_n,y_n) \}_{n=1}^N$. With the flat prior, the objective is minimizing $-\log p(\mathcalbf{D} \mid \boldsymbol{\theta },\boldsymbol{\varphi })=-\log\prod_{i=1}^Np(y_i\mid x_i,\bm\theta)
$. Considering the anisotropic uncertainty of radar measurements as shown in Fig.~\ref{level.uncertainty}, we model the likelihood with a heteroscedastic Gaussian distribution with $y_i = f_{\bm\theta}(x_i) + \epsilon_i$, where $\epsilon_i \sim \mathcal{N}(0, g_{\bm\varphi}(x_i))$. The conditional likelihood is given by:
\begin{align}
p(y_i\mid x_i,\boldsymbol{\theta },\boldsymbol{\varphi })\:=\:\frac{\exp(-||y_i-f_{\boldsymbol{\theta }}(x_i)||_{g_{\boldsymbol{\varphi }}\left( x_i \right)}^{2}/2)}{\sqrt{(2\pi )^N\det \left( g_{\boldsymbol{\varphi }}\left( x_i \right) \right)}}
\end{align}
where mahalanobis distance $||\epsilon_i||^{2}_\Sigma=\epsilon_i^T{\Sigma}^{-1}\epsilon_i$. The negative log-likelihood loss for radar uncertainty learning is:
\begin{equation}\label{nll}
	\begin{aligned}
	&\left( \hat{\boldsymbol{\theta}},\hat{\boldsymbol{\varphi}} \right) =\mathrm{arg}\min_{\boldsymbol{\theta },\boldsymbol{\varphi }}  \left[ -\log \prod_{i=1}^Np(y_i\mid x_i,\bm\theta) \right] \\
& \Rightarrow \mathrm{arg}\min_{\boldsymbol{\theta },\boldsymbol{\varphi }} \sum_{i=1}^N{\left( ||\epsilon_i(\boldsymbol{\theta })||_{g_{\boldsymbol{\varphi }}\left( x_i \right)}^{2}+\log\det \left( g_{\boldsymbol{\varphi }}\left( x_i \right) \right) \right)}
\end{aligned}
\end{equation}
The first term encourages model $\boldsymbol{\theta}$ to fit the data $\mathcalbf{D}$ while considering the predicted uncertainty $\boldsymbol{\varphi}$ from radar, and the second term penalizes large uncertainty estimates to prevent trivial solutions.

\subsection{Doppler-Aware Predictive Enhancement}\label{sub.DAPM}
Only intensity information may bring many spurious reflections due to multipath effects in radar perception. As shown in Fig.~\ref{level.doppler}, we observe the Doppler measurements of positive radar reflections conform to kinematic constraint, while most spurious reflections deviate from it. To effectively exploit Doppler consistency for enhancing radar geometric representation, we first analyze the relationship between Doppler velocity and radar intrinsic motion in Theorem~\ref{dopplervelocity}.

\begin{theorem}\label{dopplervelocity}
Doppler velocity of the reflection, i.e. relative radial velocity $v^r_{i,j}$, is determined by the radar's ego-velocity $\boldsymbol{v}^r$ and the scatter's directional vector $(\alpha,\beta)$ when the target is stationary relative to the ground, where $(\alpha,\beta)$ denotes the azimuth and elevation of the scatter with respect to the radar's coordinate system.
\end{theorem}
\begin{proof}
\begin{figure}[H] 
	\centering  
	\includegraphics[width=.86\linewidth]{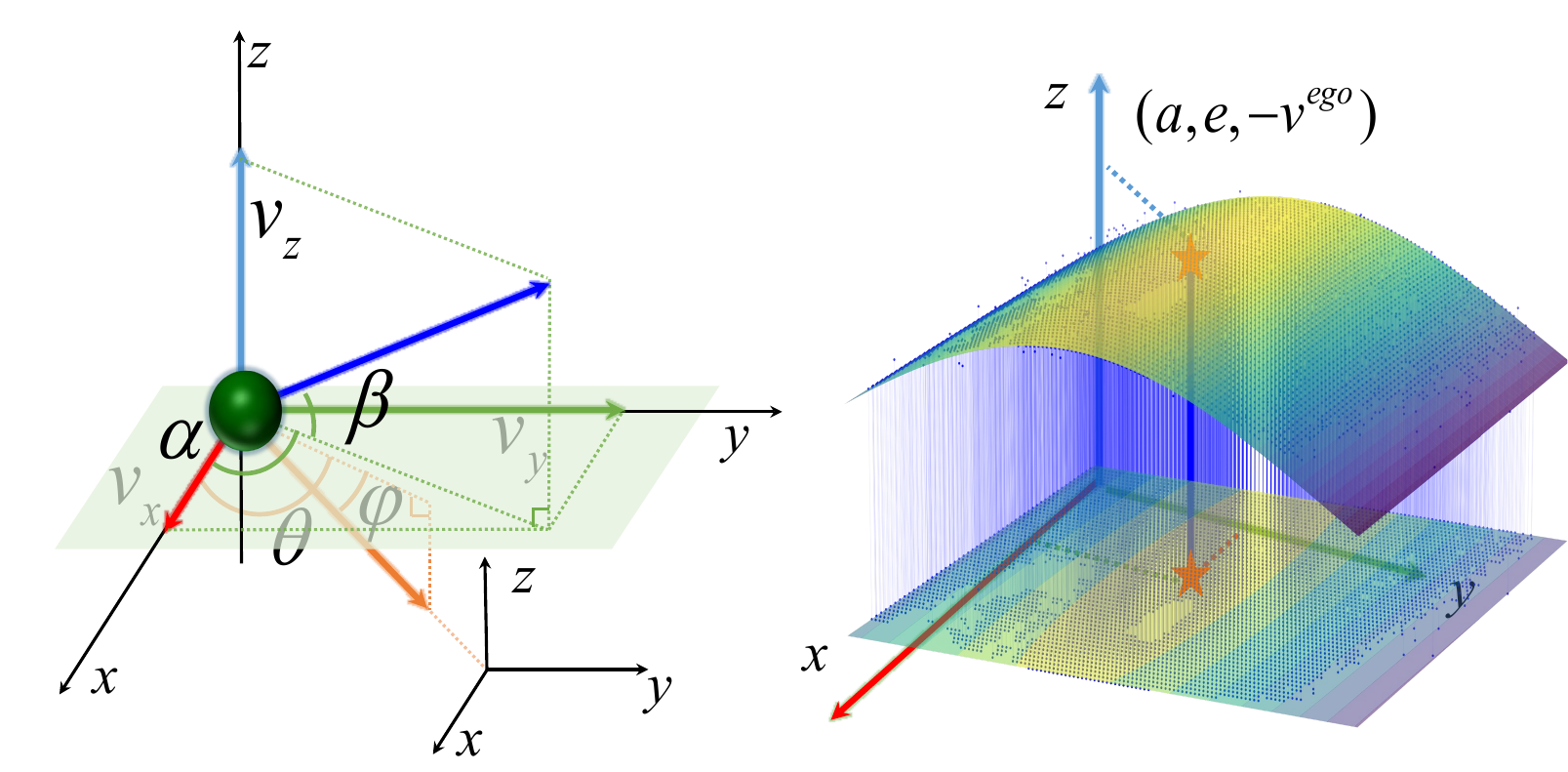}
	\caption{Doppler-Consistency for stationary positive scatters.}
	\label{fig:doppler_consistency}
\end{figure}
Fig.~\ref{fig:doppler_consistency} illustrates the principle of Doppler-Consistency. The detailed proof and technical details are available at \url{https://shengpeng.wang/rauf/#Supplementary}.
\end{proof}
Therefore, spurious reflections can be effectively suppressed through Doppler consistency, as established in Theorem~\ref{dopplervelocity}. To this end, we design a Doppler-Aware Predictive Enhancement Module, as illustrated in Fig.~\ref{fig:overview}, which reinforces spatial features with Doppler information via the proposed Bidirectional Domain Attention Fusion~(BDAF) described in Sec.~\ref{sub.Network}.

\subsection{Network \& Uncertainty Training}\label{sub.Network}
\textbf{Radar Feature Extraction.} Given the radar measurement input $x$ including intensity and Doppler channels, we first extract spatial features ${\bm F}_s \in \mathbb{R}^{H \times W\times C_s}$ encoding spatial occupancy and Doppler features ${\bm F}_d \in \mathbb{R}^{H \times W\times C_d}$ reflecting velocity distribution using two separate 3D CNN backbones. Unlike previous methods~\cite{10592769,prabhakara2023high} that either collapse elevation into a feature channel or use 2D convolutions along range and azimuth, our 3D CNN backbone jointly models range, azimuth, and elevation correlations, enabling richer spatial representations and improved capture of radar scattering patterns.

\noindent \textbf{Bidirectional Domain Attention Fusion.} The intensity map provides reliable spatial localization cues, while the Doppler map helps suppress spurious reflections. To effectively exploit Doppler consistency for enhancing radar geometric representation, we propose a Bidirectional Domain Attention Fusion~(BDAF) module for doppler radar to reinforce spatial features with Doppler information. As illustrated in Fig.~\ref{fig:overview}, the BDAF consists of two cross-attention layers that model the correlation between spatial and Doppler features. Specifically, both features are first patchified and embedded with positional encodings, forming sequences ${\bm S}_p, {\bm D}_p \in \mathbb{R}^{L \times C}$ to enable token-wise interaction under a unified representation length $L=\frac{HW}{p^2}$. The first stage of BDAF enhances Doppler features by leveraging spatial cues as attention queries with ${\bm Q_s}={\bm W}_s^q{\bm S}_p$, ${\bm K}_d={\bm W}_d^k{\bm D}_p$, and ${\bm V}_d={\bm W}_d^v{\bm D}_p$, where ${\bm W}_s^q$, ${\bm W}_d^k$, and ${\bm W}_d^v$ are learnable projection matrices. The attention output is computed as:
\begin{equation}
	{\bm S}_p\rightarrow {\bm D}_p^{\prime} =\mathrm{Softmax}\left( \frac{{\bm Q_s}{\bm K}_d^T}{\sqrt{d_k}} \right) {\bm V}_d
\end{equation}
where $d_k$ is the dimension of the key vectors. This allows the network to focus on regions where Doppler measurements are consistent with the radar's ego-motion. ${\bm D}_p^{\prime}$ is then projected into Doppler features $\tilde{\bm D}_p$ after concatenation with the original Doppler features ${\bm D}_p$ through residual networks. To further align Doppler information with spatial occupancy semantics, we introduce a Domain Projection network that transforms $\tilde{\bm D}_p$ into an occupancy-like latent representation ${\bm D}_p^{\prime \prime}$ before feeding it into the second cross-attention layer, which is implemented as a lightweight residual bottleneck. The output encodes spatial likelihood from Doppler distributions, serving as a differentiable physical prior for subsequent attention refinement.

In turns, the second stage reverses this process to enhance spatial features using Doppler-projected tokens as  queries to guide spatial feature reconstruction with ${\bm S}_p^{\prime}=Attention({\bm Q_d}={\bm W}_d^q{\bm D}_p^{\prime \prime}, {\bm K}_s={\bm W}_s^k{\bm S}_p, {\bm V}_s={\bm W}_s^v{\bm S}_p)$. The final enhanced spatial representation is obtained as ${\bm F}_{s}^{\prime} = \mathrm{ResNet}(\mathrm{Concat}({\bm S}_p, {\bm S}_p^{\prime}))$. This bidirectional refinement enables mutual compensation between spatial precision and Doppler coherence and enhances focus on consistent Spatial-Doppler patterns. The output ${\bm F}_{s}^{\prime}$, $\tilde{\bm D}_p$ are subsequently fed into feature decoders for spatial NLL loss $\mathcal{L} _{\mathrm{spatial}}$ and velocity regression loss $ \mathcal{L} _{\mathrm{doppler}}$, respectively.

\noindent \textbf{Anisotropic Uncertainty Learning for Radar.} To model the anisotropic uncertainty of radar measurements as shown in Fig.~\ref{level.uncertainty}, we design the uncertainty quantification model $g_{\bm\varphi}(x)$ to predict a heteroscedastic Gaussian distribution with an anisotropic covariance matrix. Specifically, we parameterize the uncertainty in polar coordinates, predicting radial uncertainty $\sigma_t$ and angular uncertainty $\sigma_a$ separately. The approximate uncertainty distribution in the Cartesian coordinate system is derived through first-order error propagation as detailed in Theorem~\ref{anisotropic_uncertainty}.

\begin{theorem}\label{anisotropic_uncertainty}
For a radar detection measurement ${\bm {\hat p}}_{i,j,k}$ truely located at $(r_i, \alpha_j, \beta_k)$ in polar coordinates with radial range uncertainty $\boldsymbol{\delta }_{r _i}\sim \mathcal{N} \left( 0,\sigma^2_{r_i} \right) $, azimuth uncertainty $\boldsymbol{\delta }_{\alpha _i}\sim \mathcal{N} \left( 0,\sigma^2_{\alpha_i} \right) $ and elevation uncertainty $\boldsymbol{\delta }_{\beta _i}\sim \mathcal{N} \left( 0,\sigma^2_{\beta_i} \right) $, the corresponding uncertainty in Cartesian coordinates can be approximated by a Gaussian distribution.
\end{theorem}
\begin{proof}
Considering $\bm p$ can be approximated by a first-order Taylor expansion around $(r, \alpha, \beta)$, we have:
\begin{equation}
\begin{aligned}
&\boldsymbol{\delta }_{\boldsymbol{p}}=\frac{\partial \boldsymbol{p}}{\partial r}\boldsymbol{\delta }_r+\frac{\partial \boldsymbol{p}}{\partial \alpha}\boldsymbol{\delta }_{\alpha}+\frac{\partial \boldsymbol{p}}{\partial \beta}\boldsymbol{\delta }_{\beta}\\
&=\left[ \begin{matrix}
	\cos \alpha \cos \beta&		-r\sin \alpha \cos \beta&		-r\cos \alpha \sin \beta\\
	\sin \alpha \cos \beta&		r\cos \alpha \cos \beta&		-r\sin \alpha \sin \beta\\
	\sin \beta&		0&		r\cos \beta\\
\end{matrix} \right] \left[ \begin{array}{c}
	\boldsymbol{\delta }_r\\
	\boldsymbol{\delta }_{\alpha}\\
	\boldsymbol{\delta }_{\beta}\\
\end{array} \right]
\end{aligned}
\end{equation}
Let the Jacobian matrix denoted as ${\bm J}$. Since $\boldsymbol{\delta }_r$, $\boldsymbol{\delta }_{\alpha}$, and $\boldsymbol{\delta }_{\beta}$ are independent Gaussian variables, $\boldsymbol{\delta }_{\boldsymbol{p}}$ is also a Gaussian variable with covariance . Therefore, the uncertainty in Cartesian coordinates can be approximated by a Gaussian distribution $\mathcal{N}(0,{\bm \Sigma})$, with covariance matrix $\mathbf{\Sigma }=\boldsymbol{JDJ}^T$, where $\mathbf{D}=\mathrm{diag}(\sigma^2_r,\sigma^2_{\alpha},\sigma^2_{\beta})$.
\end{proof}

According to Theorem~\ref{dopplervelocity}, ``crescent-shaped'' uncertainty in radar measurements can be approximated by ``ellipsoid-shaped'' confidence with an anisotropic Gaussian distribution as shown in Fig.~\ref{level.uncertainty}. Combining Eqn.~\ref{nll} and Theorem~\ref{anisotropic_uncertainty}, we optimize the spatial localization model $f_{\bm\theta}(x)$ and uncertainty model $\mathbf{\Sigma }=g_{\bm\varphi}(x)$ by minimizing the negative log-likelihood loss with anisotropic covariance matrix as:
\begin{equation}\label{finalnll}
	\mathcal{L} _{\mathrm{spatial}}=\sum_{i=1}^N{\left( ||\epsilon _i(\boldsymbol{\theta })||_{g_{\boldsymbol{\varphi }}\left( x_i \right)}^{2}+\log\det \left( g_{\boldsymbol{\varphi }}\left( x_i \right) \right) \right)}
\end{equation}
Besides, to further enhance Doppler consistency, we introduce an auxiliary velocity regression task that predicts the Doppler velocity for each radar detection.

\noindent \textbf{Frustum-based Voxelization Strategy.} To effectively supervise radar uncertainty learning, we adopt a frustum-based voxelization strategy to construct occupancy ground truth from LiDAR PCLs. Given the radar's intrinsic parameters, we generate frustum-shaped regions by extending rays along range, azimuth, and elevation. All LiDAR points within this frustum are considered as potential reflections corresponding to radar detections. This approach allows us to create a more accurate occupancy ground truth that reflects the inherent uncertainty in radar measurements.

\section{Experiments and Evaluation}
We evaluate the proposed method using both the publicly available Coloradar and Radelft datasets, as well as a self-collected dataset across various indoor and outdoor scenarios with different types of radar sensors.

\subsection{Dataset}
\noindent \textbf{Radelft Dataset.}
Radelft dataset~\cite{iroldan2024} provides synchronized radar analog-to-digital converter~(ADC) data and LiDAR data collected in large-scale urban environments. It uses a long-range cascade radar with maximum detection range of 51 meters to record data for about 35 minutes in 7 different scenes.

\noindent \textbf{ColoRadar Dataset.} It contains 43,000 radar frames collected in various indoor and outdoor environments using a short-range single radar with 15 m maximum detection range and a cascade radar. We use LiDAR point clouds as the ground truth for PCL extraction estimation and velocity solved through odometry for velocity supervision.

\noindent \textbf{Self-collected Dataset.}
We collected real-world data from diverse environments using our platform to further validate the model's generalization. The dataset includes single-chip and cascade radar data in a total of over 11,000 frames across various indoor and outdoor scenarios. Reliable velocity ground truth is obtained using Fast-Livo2~\cite{Fastlivo2}.

\begin{figure*}[t]
  \centering
   \includegraphics[width=\linewidth]{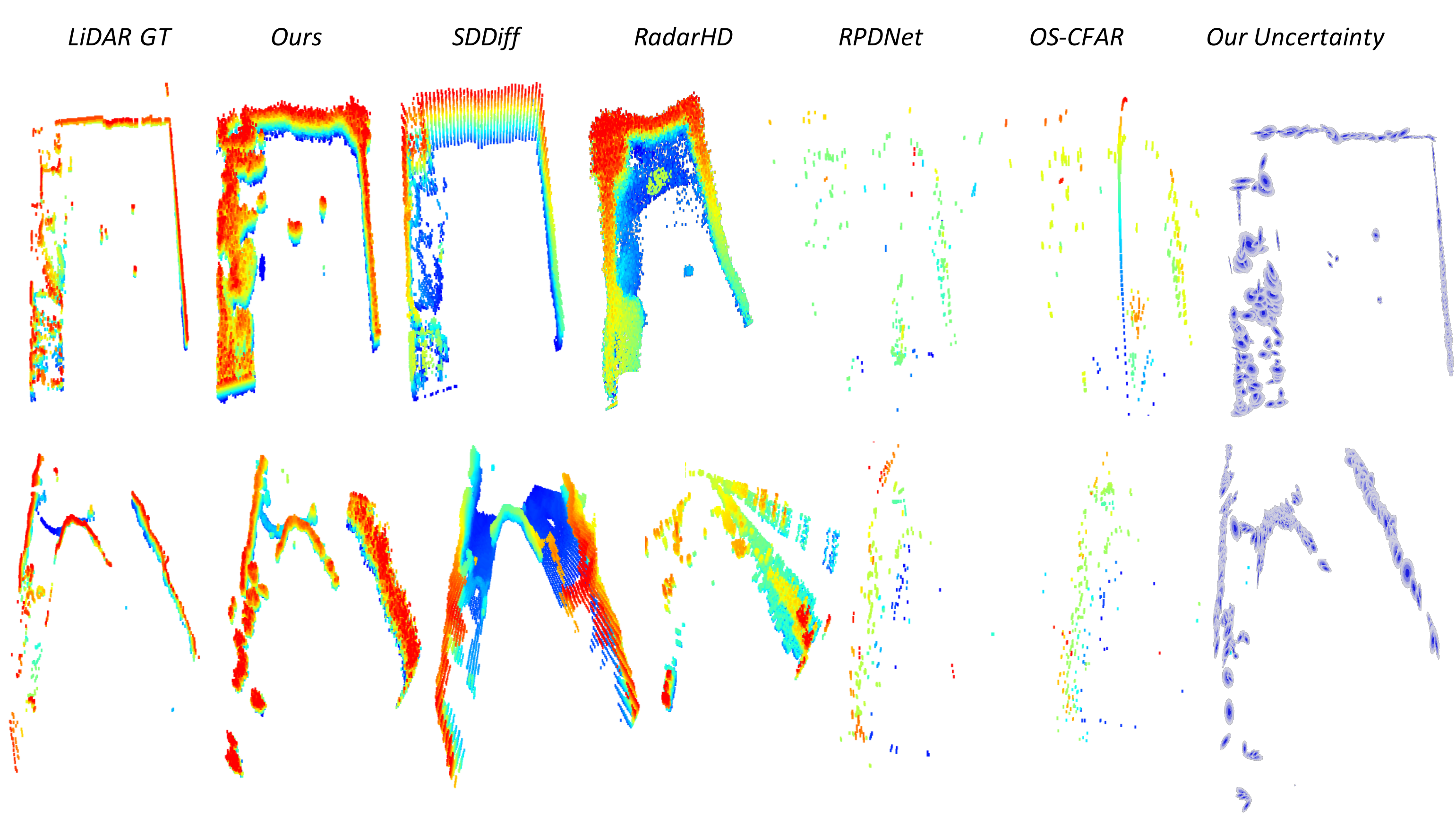}
   \caption{Bird’s-Eye View visualization of Radar point clouds augmented by various models and LiDAR ground truth.}
   \label{fig:gallery}
\end{figure*}

\noindent \textbf{Implementation Details.}
For the negative log-likelihood (NLL) loss defined in Eqn.~\ref{finalnll}, we predict the exponential of the variance term to enforce positivity, thereby mitigating numerical instability during model training. For supervising occupancy and velocity estimates, we employ an MSE loss with a weighting coefficient of 0.001. We utilize the AdamW optimizer with an initial learning rate of $2 \times 10^{-4}$ for training. All experiments are conducted on a workstation equipped with four NVIDIA GeForce RTX 4090 GPUs and an Intel Xeon Gold 6226R CPU, with the total training time amounting to 6 days.

\subsection{Baseline \& Case Studies}
RaUF is the first framework to jointly reconstruct radar point clouds and explicitly estimate their spatial uncertainties. To evaluate its effectiveness and reliability, we compare it with traditional model-based methods OS-CFAR~\cite{blake1988cfar} and RPDNet~\cite{cheng2022novel}, as well as state-of-the-art learning-based methods RadarHD~\cite{prabhakara2023high} and SDDiff~\cite{ijcai2025p979}. 
Furthermore, case studies on transformation estimation~(TE) and ego-velocity estimation~(EVE) are conducted against ICP~\cite{besl1992method}, GICP~\cite{segal2009generalized}, RANSAC~\cite{fischler1981random}, and RadarEVE~\cite{pang2024radarmoseve}, confirming that uncertainty estimation substantially enhances downstream perception robustness and accuracy.

\subsection{Evaluation Metric}

\noindent \textit{1) Effectiveness:}
We employ Chambers Distance~(CD) and F-score~(FS) to evaluate the representational effectiveness of reconstructed point clouds $\mathcal{P} $.

\noindent \textit{2) Reliability:}
To evaluate the reliability of generated PCLs against clutters $\mathcal{P} _{\mathrm{c}}$, we define Clutter Point Ratio~(CPR) $\eta$ calculated as follows:
\begin{align}
    \eta =\frac{|\mathcal{P} _{\mathrm{c}}|}{|\mathcal{P} |},\mathcal{P} _{\mathrm{c}}=\{p\in \mathcal{P} ,\text{s.t.}, d\left( q,p \right) >\zeta ,\:\forall q\in \mathcal{Q} \}
\end{align}
where $\mathcal{Q} $ is the set of ground truth PCLs, $d(\cdot ,\cdot )$ is the Euclidean distance function between two points, and $\zeta $ is a distance threshold set to 0.5m in our experiments. A lower CPR indicates more reliable PCLs with fewer clutter points.

\noindent \textit{3) Scalability:} To validate the practical utility of the proposed uncertainty estimation, we assess its impact on downstream tasks to highlight the scalability and general applicability of our approach across diverse radar perception tasks.

\begin{table*}[t]
\centering
\begin{minipage}{0.6\textwidth}
\centering
\fontsize{7}{7}\selectfont
\setlength{\tabcolsep}{0.5pt}
\begin{tabular}{
    c|ccccccccccccc
}
\toprule[0.5pt]
\toprule
\multirow{2}{*}{\makecell{Radar \\ Category}} &
\multirow{2}{*}{\makecell{Detection \\ Methods}} &
\multicolumn{2}{c}{Armyroom} &
\multicolumn{2}{c}{Classroom} &
\multicolumn{2}{c}{Aspen Room} &
\multicolumn{2}{c}{Labroom} &
\multicolumn{2}{c}{Hallways} &
\multicolumn{2}{c}{Longboard} \\
\cmidrule(lr){3-4}
\cmidrule(lr){5-6}
\cmidrule(lr){7-8}
\cmidrule(lr){9-10}
\cmidrule(lr){11-12}
\cmidrule(lr){13-14}
& & CD $\downarrow $ & FS $\uparrow $ & CD $\downarrow $ & FS $\uparrow $ & CD $\downarrow $ & FS $\uparrow $ & CD $\downarrow $ & FS $\uparrow $ & CD $\downarrow $ & FS $\uparrow $ & CD $\downarrow $ & FS $\uparrow $  \\

\midrule
\multirow{6}{*}{\rotatebox[origin=c]{90}{\makecell{Coloradar \\ Single-Chip}}} &
{OS-CFAR}~\cite{blake1988cfar} 
& \acrot{4.15} & \actrans{0.19} & \acrot{4.81} & \actrans{0.17} & \acrot{2.99} & \actrans{0.02}  & \acrot{2.78} & \actrans{0.02} & \acrot{2.75} & \actrans{0.02} & - & - \\
& {RPDNet}~\cite{cheng2022novel} 
& \acrot{4.07} & \actrans{0.18} & \acrot{5.83} & \actrans{0.13} & \acrot{2.46} & \actrans{0.03}  & \acrot{2.45} & \actrans{0.03} & \acrot{2.64} & \actrans{0.03} & - & - \\
& {RadarHD}$^\dag$~\cite{prabhakara2023high} 
& \acrot{1.74} & \actrans{0.42} & \acrot{2.45} & \actrans{0.28} & \acrot{0.95} & \actrans{0.23}  & \acrot{0.87} & \actrans{0.25} & \acrot{1.24} & \actrans{0.18} & - & - \\
& SDDiff~\cite{ijcai2025p979} 
& \acrot{0.42} & \actrans{0.17} & \acrot{5.38} & \actrans{0.17} & \acrot{3.78} & \actrans{0.02}  & \acrot{3.97} & \actrans{0.02} & \acrot{3.70} & \actrans{0.02} & - & -  \\
& \textbf{Ours} (w.o. NLL) 
& \acrot{2.62} & \actrans{0.53} & \acrot{6.35} & \actrans{0.31} & \acrot{0.58} & \actrans{0.38}  & \acrot{0.86} & \actrans{0.32} & \acrot{1.34} & \actrans{0.26} & - & -  \\
& \textbf{Ours} (w.o. BDA) 
& \acrot{2.00} & \actrans{0.61} & \acrot{3.44} & \actrans{0.47} & \acrot{0.54} & \actrans{0.36}  & \acrot{0.76} & \actrans{0.33} & \acrot{1.14} & \actrans{0.28} & - & -  \\
& \textbf{Ours} (w.o. GS)
& \acrot{1.33} & \actrans{0.69} & \acrot{1.89} & \actrans{0.56} & \acrot{0.48} & \actrans{0.46}  & \acrot{0.66} & \actrans{0.38} & \acrot{0.79} & \actrans{0.35} & - & -  \\
& \textbf{Ours} 
& \acrot{1.31} & \actrans{0.69} & \acrot{1.74} & \actrans{0.54} & \acrot{0.48} & \actrans{0.44}  & \acrot{0.64} & \actrans{0.38} & \acrot{0.78} & \actrans{0.34} & - & -  \\
\midrule

\renewcommand{\acrot}[1]{%
  \pgfmathsetmacro\percentcolor{max(0,min(100,#1/4.0*100))}%
  \edef\tempcolor{\noexpand\cellcolor{green!\percentcolor!white}}%
  \tempcolor\makebox[1.5em][c]{#1}%
}
\renewcommand{\actrans}[1]{%
  \pgfmathsetmacro\percentcolor{max(0,min(100,#1/1.0*100))}%
  \edef\tempcolor{\noexpand\cellcolor{orange!\percentcolor!white}}%
  \tempcolor\makebox[1.5em][c]{#1}%
}

\multirow{6}{*}{\rotatebox[origin=c]{90}{\makecell{Coloradar \\ Cascade}}} &
{OS-CFAR}~\cite{blake1988cfar} 
& \acrot{2.14} & \actrans{0.04} & \acrot{1.95} & \actrans{0.05} & \acrot{2.04} & \actrans{0.02}  & \acrot{2.20} & \actrans{0.04} & \acrot{2.19} & \actrans{0.06} & \acrot{12.99} & \actrans{0.04}   \\
& {RPDNet}~\cite{cheng2022novel} 
& \acrot{1.81} & \actrans{0.05} & \acrot{1.39} & \actrans{0.08} & \acrot{1.82} & \actrans{0.06}  & \acrot{1.78} & \actrans{0.16} & \acrot{1.84} & \actrans{0.15} &\acrot{13.75} & \actrans{0.02} \\
& {RadarHD}$^\dag$~\cite{prabhakara2023high} 
& \acrot{1.08} & \actrans{0.23} & \acrot{1.29} & \actrans{0.25} & \acrot{1.49} & \actrans{0.14}  & \acrot{1.12} & \actrans{0.22} & \acrot{1.38} & \actrans{0.19} &\acrot{4.66} & \actrans{0.19} \\
& SDDiff~\cite{ijcai2025p979} 
& \acrot{0.86} & \actrans{0.43} & \acrot{0.69} & \actrans{0.44} & \acrot{1.16} & \actrans{0.38}  & \acrot{1.97} & \actrans{0.12} & \acrot{1.92} & \actrans{0.15} & \acrot{9.00} & \actrans{0.08}  \\
& \textbf{Ours} (w.o. NLL) 
& \acrot{0.58} & \actrans{0.45} & \acrot{0.48} & \actrans{0.50} & \acrot{0.83} & \actrans{0.38}  & \acrot{1.22} & \actrans{0.31} & \acrot{1.20} & \actrans{0.33} & \acrot{7.97} & \actrans{0.20}  \\
& \textbf{Ours} (w.o. BDA) 
& \acrot{0.99} & \actrans{0.40} & \acrot{0.50} & \actrans{0.41} & \acrot{0.81} & \actrans{0.37}  & \acrot{1.15} & \actrans{0.34} & \acrot{1.11} & \actrans{0.36} & \acrot{5.98} & \actrans{0.28}  \\
& \textbf{Ours} (w.o. GS) 
& \acrot{0.50} & \actrans{0.49} & \acrot{0.45} & \actrans{0.54} & \acrot{0.76} & \actrans{0.40}  & \acrot{1.17} & \actrans{0.33} & \acrot{1.15} & \actrans{0.35} & \acrot{3.86} & \actrans{0.37}  \\
& \textbf{Ours} 
& \acrot{0.50} & \actrans{0.47} & \acrot{0.45} & \actrans{0.54} & \acrot{0.75} & \actrans{0.39}  & \acrot{1.12} & \actrans{0.34} & \acrot{1.10} & \actrans{0.36} &\acrot{3.79} & \actrans{0.36}  \\

\bottomrule[0.5pt]
\end{tabular}
\caption{Comparison of spatial detection performance across different scenes and radar categories on the Coloradar dataset.}
\vspace{0.5ex}
{\footnotesize * The Longboard scene (Single Chip) is excluded due to extremely few valid points after ground removal and FOV-constrained filtering.}
\label{tab:odd_comparison}
\end{minipage}\hfill 
\begin{minipage}{0.38\textwidth}
\centering
    \fontsize{7}{7}\selectfont
    \setlength{\tabcolsep}{0.3pt}
    \begin{tabular}{cccccccc}
        \toprule[0.5pt]
        \toprule
        \multirow{2}{*}{\makecell{Methods}} &
        \multicolumn{5}{c}{EVE Error Distribution (m/s)} &
        \multicolumn{2}{c}{TE Error} \\
        \cmidrule(lr){2-6}
        \cmidrule(lr){7-8}
       & 10\% $\downarrow $& 20\% $\downarrow $& 40\% $\downarrow $& 60\% $\downarrow $ & 80\% $\downarrow $& T.(m) $\downarrow $ & R.(deg) $\downarrow $ \\
        \cmidrule(lr){1-8}
        ICP & 0.15 & 0.20 & 0.41 & 0.65 & 1.03 & 0.99 & 0.63 \\
        GICP  & 0.13    & 0.27  & 0.38  & 0.57   & 1.26  & \textbf{0.54} & \textbf{0.50} \\
       RadarEVE  & \uline{\textbf{0.09}}  & \uline{\textbf{0.15}}     & \textbf{0.26}           & \uline{\textbf{0.37}}           & \uline{\textbf{0.75}}   & \XSolidBrush & \XSolidBrush       \\
    RANSAC & \textbf{0.11} & 0.23 & 0.32 & 0.56 & 0.85 & 0.86 & 0.65\\
        \textbf{Ours} & \uline{\textbf{0.09}}  & \textbf{0.17} & \uline{\textbf{0.25}} & \textbf{0.39} & \textbf{0.77} & \uline{\textbf{0.42}} & \uline{\textbf{0.39}}\\
        \bottomrule
        \bottomrule[0.5pt]
    \end{tabular}
    \caption{Evaluation of Case Studies on EVE and TE.}
    \label{tab:case}
    \fontsize{7}{7}\selectfont
    \setlength{\tabcolsep}{3.8pt}
    \begin{tabular}{ccccc}
        \toprule[0.5pt]
        \toprule
        \multirow{2}{*}{\makecell{Detection\\ Methods}} &
        \multicolumn{2}{c}{RaDelft Dataset} &
        \multicolumn{2}{c}{Self-Collected}  \\

        \cmidrule(lr){2-3}
        \cmidrule(lr){4-5}
        & Seq 2 & Seq 4 & Indoor & Outdoor \\

        \midrule
        OS-CFAR & 0.19~/~0.76 & 0.17~/~0.78 & 0.15~/~0.77 & 0.39~/~0.87 \\
        RPDNet  & 0.22~/~0.75 & 0.15~/~0.72 & 0.22~/~0.70 & 0.17~/~0.74 \\
        RadarHD & 0.42~/~0.58 & 0.28~/~0.59 & 0.40~/~0.56 & 0.39~/~0.66 \\
        SDDiff  & \textbf{0.66}~/~\textbf{0.44} & \textbf{0.46}~/~\textbf{0.51} & \textbf{0.43}~/~\textbf{\uline{0.44}} & \textbf{0.43}~/~\uline{\textbf{0.56}} \\
        \textbf{Ours} & \textbf{\uline{0.69}}~/~\uline{\textbf{0.41}} & \uline{\textbf{0.54}}~/~\uline{\textbf{0.50}} & \textbf{\uline{0.45}}~/~\textbf{\uline{0.44}} & \uline{\textbf{0.49}}~/~\textbf{0.61} \\

        \bottomrule[0.5pt]
    \end{tabular}
    \caption{Evaluation of reliability, representational effectiveness, and computational efficiency. (Left: F-score; Right: CPR.)}
    \label{tab:radelft}
    \end{minipage}
\end{table*}

\subsection{Performance and Analysis}
\noindent \textbf{Spatial Detection.} Qualitatively, we present the spatial detection results of different approaches for both cascade and single-chip radars across various scenes, as shown in Fig.~\ref{fig:gallery}. Our method achieves higher point density and richer structural fidelity than existing alternatives, and performs on par with the state-of-the-art SDDiff. Quantitatively, we compare point clouds generated by different methods on the Coloradar dataset against LiDAR PCLs. As shown in Tab.~\ref{tab:odd_comparison}, our method achieves more accurate prediction, outperforming traditional CFAR by 70.1\% and 5$\times$ on average CD and F-score, demonstrating superior geometric fidelity. Some baselines focus on full-scene reconstruction including ground features, whereas we perform evaluations on ground-free point clouds. This setup better aligns with the requirements of most downstream perception tasks.

In addition, we evaluate long-range cascade radar performance on the RaDelft dataset, assessing both effectiveness and robustness using CD and CPR, as well as runtime efficiency in Tab.~\ref{tab:radelft}. Our model maintains strong performance when fine-tuned on previously unseen scenes, demonstrating robust generalization capability. The results demonstrate the superiority of our method. To further validate model generalization, we perform fine-tuning on our self-collected dataset. RaUF adapts rapidly to varying sensor configurations, demonstrating robustness. This indicates that our uncertainty-quantification network can calibrate the confidence of predicted samples, effectively mitigating the issues arising from ill-posed geometric inference in low-fidelity space. Additional experimental results and scenario visualizations are provided in the supplementary material.

\begin{figure*}[h]
    \centering
    \includegraphics[width=\linewidth]{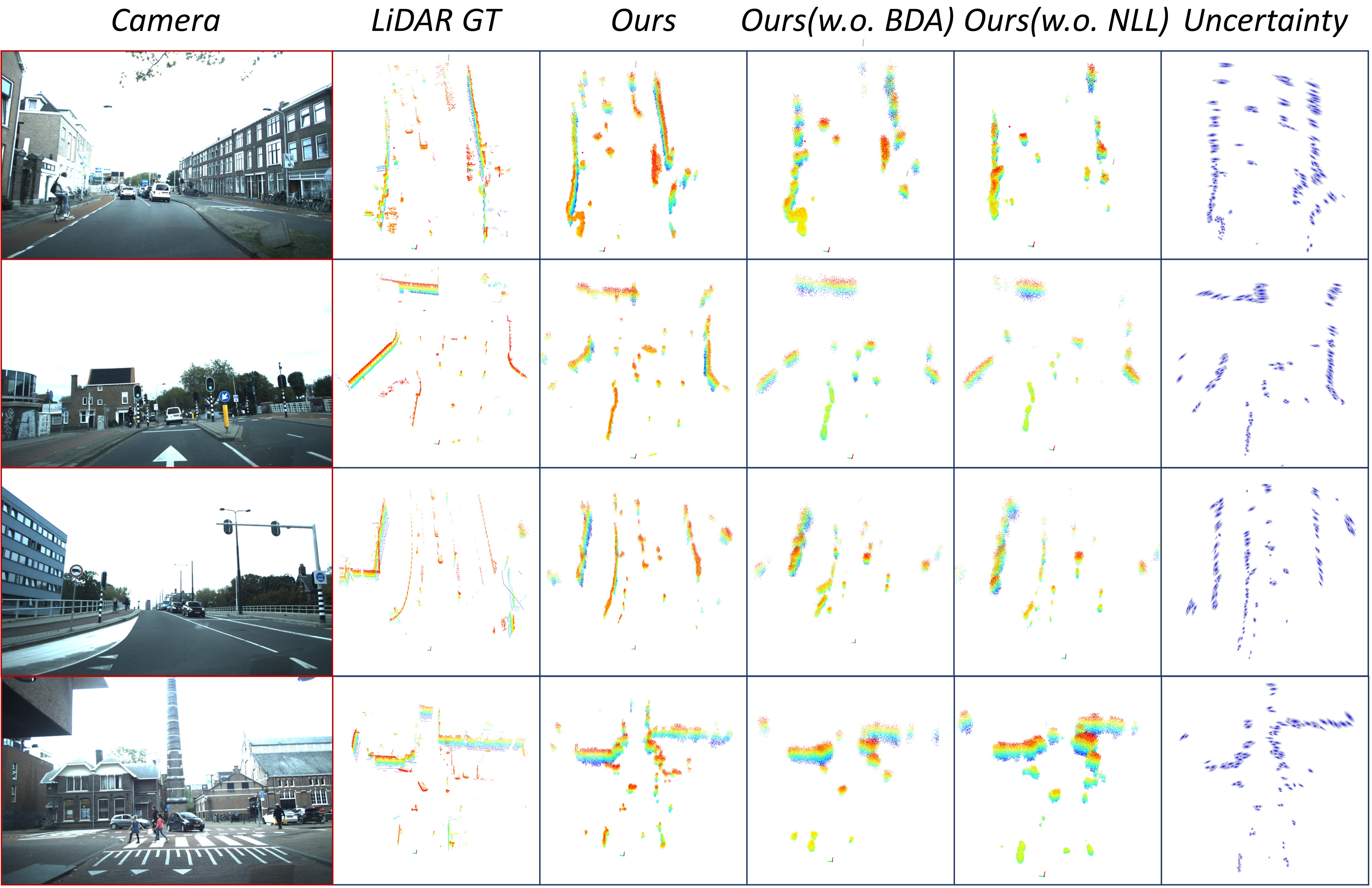}
    \caption{Bird’s-Eye View visualization of Radar point clouds augmented by various methods and LiDAR ground truth.}
    \label{fig:bev}
\end{figure*}

\noindent \textbf{Uncertainty quantification.} As illustrated in the visualizations, our model effectively leverages physically grounded radar mechanisms to capture the characteristic crescent-shaped spatial distribution. This not only provides reliable confidence estimates over target regions but also serves as a valuable prior for downstream tasks, including 3D reconstruction, localization, object tracking, and velocity estimation, by guiding geometrically consistent reasoning and robust measurement utilization.

To evaluate the scalability of our uncertainty quantification for downstream tasks, we conduct case studies on transformation estimation and ego-velocity prediction. Quantitative results are summarized in Tab.~\ref{tab:case}. Using our uncertainty-aware point clouds, we achieve the best transformation-estimation performance, improving over GICP by 22\% in both translational and rotational accuracy. This gain stems from the fact that GICP’s covariance estimation relies solely on local geometric statistics, which fail to capture the physically meaningful uncertainty intrinsic to radar measurements. This highlights the advantage of our learned uncertainty representation over traditional structure-based estimators. For ego-velocity estimation, our method surpasses classical RANSAC and performs on par with the end-to-end RadarEVE. These results further demonstrate that our Bidirectional Domain Attention mechanism facilitates mutual compensation between spatial precision and Doppler coherence, thereby enhancing the stability and physical consistency of motion inference.

\subsection{Ablation Study}
We conduct an ablation study to assess the effectiveness of uncertainty calibration in mitigating ill-posed geometric inference in low-fidelity radar spaces, as well as the contribution of the proposed Bidirectional Domain Attention~(BDA). As shown in Tab.~\ref{tab:odd_comparison} and Fig.~\ref{fig:bev}, incorporating uncertainty calibration reduces CD by 30.55\%-37.18\% and improves F-score by 27.27\%-38.71\%, indicating that modeling anisotropic uncertainty provides a more physically consistent and conflict-free geometric representation. In addition, the BDA module improves the intensity-only baseline by 14.98\%-15.95\%, demonstrating that BDA effectively leverages the mutual complementarity between spatial structure and Doppler coherence, thereby enhancing feature reliability under cluttered conditions.
\section{Conclusion}
In this work, we presented a spatial uncertainty field learning framework for radar perception that fundamentally predicts spatial detections with anisotropic uncertainty. 
To mitigate ill-posed geometric inference in low-fidelity radar space, we design an anisotropic probabilistic model for uncertainty learning. Additionally, we design Bidirectional Domain Attention that effectively exploits the mutual complementarity between spatial structure and Doppler coherence, enabling the model to suppress spurious scatters. Comprehensive evaluations across public and real-world datasets confirm the strong  representational effectiveness, reliability, and scalability of our approach on downstream tasks.

\section*{Acknowledgments}
This work was supported in part by National Natural Science Foundation of China with Grant 62522214, 62471194, Key Research and Development Program of Hubei Province under grant number 2025BAB023.
{
    \small
    \bibliographystyle{ieeenat_fullname}
    \bibliography{main}
}


\end{document}